\documentclass[letterpaper, 10 pt, conference]{ieeeconf}
\usepackage{cite}
\usepackage{graphicx}
\usepackage{longtable,tabularx}
\usepackage{caption}
\usepackage{booktabs}
\usepackage{authblk}
\usepackage{caption}
\usepackage{algorithm}
\usepackage{algpseudocode}
\usepackage{multirow}
\usepackage{balance}
\usepackage{amsmath,amsfonts}
\usepackage{array}
\usepackage[caption=false,font=normalsize,labelfont=sf,textfont=sf]{subfig}
\usepackage{textcomp}
\usepackage{stfloats}
\usepackage[T1]{fontenc}
\usepackage{longtable}
\usepackage{xhfill}
\usepackage{booktabs}
\usepackage{makecell}
\usepackage{scalerel}

\def\p(#1|#2){p(#1\,|\,#2)}
\def\q(#1|#2){q(#1\,|\,#2)}

\usepackage{xcolor}

\newtheorem{theorem}{Theorem}

\newcommand{\norm}[1]{\left\lVert#1\right\rVert}

\algnewcommand{\Inputs}[1]{%
  \State \textbf{Inputs:} 
   \hspace*{0.3em}\parbox[t]{\linewidth}{\raggedright #1}
}
\algnewcommand{\Initialize}[1]{%
  \State \textbf{Initialize:}
   \hspace*{0.3em}\parbox[t]{.8\linewidth}{\raggedright #1}
}
\algnewcommand{\Output}[1]{%
  \State \textbf{Outputs:}
   \hspace*{0.3em}\parbox[t]{.8\linewidth}{\raggedright #1}
}

\usepackage{xspace}

\newcommand{\EnKMP}{{\tt Ensemble Kalman Motion Planner}\xspace}
\renewcommand{\baselinestretch}{0.97}

\begin{document}

\title{\bf{Motion Planning for Autonomous Vehicles: When Model Predictive Control Meets Ensemble Kalman Smoothing}}
\IEEEoverridecommandlockouts
\author[1]{Iman Askari\textsuperscript{1}, Yebin Wang\textsuperscript{2}, Vedang M. Deshpande\textsuperscript{2}, and Huazhen Fang\textsuperscript{1}
\thanks{This work was sponsored in part by the U.S. Army Research Laboratory under Cooperative Agreement Number W911NF-22-2-0207.}
\thanks{\textsuperscript{1}I. Askari and H. Fang are with the Department of Mechanical Engineering, University of
Kansas, Lawrence, KS 66045, USA. (Email: \tt\small \{askari,fang\}@ku.edu)}
\thanks{\textsuperscript{2}Y. Wang, and V. M. Deshpande are with the Mitsubishi Electric Research Laboratories, Cambridge, MA 02139, USA. (Email: \tt \small \{yebinwang,deshpande\}@merl.com)}

}

\maketitle

\begin{abstract} 

Safe and efficient motion planning is of fundamental importance for autonomous vehicles. This paper investigates motion planning based on nonlinear model predictive control (NMPC) over a neural network vehicle model. We aim to overcome the high computational costs that arise in NMPC of the neural network model due to the highly nonlinear  and  nonconvex optimization.  In a departure from numerical optimization solutions, we reformulate the problem of NMPC-based motion planning as a Bayesian estimation problem, which seeks to infer optimal planning decisions from planning objectives. Then, we use a sequential ensemble Kalman smoother to accomplish the estimation task, exploiting its high computational efficiency for complex nonlinear systems. The simulation results show an improvement in computational speed by orders of magnitude, indicating the potential of the proposed approach for practical motion planning.
\end{abstract} 

\section{Introduction} 
\label{introduction}

Autonomous driving technologies have gained rapid advances in the past decade, showing significant promise in improving safety in transportation and access to mobility for diverse groups of people~\cite{Woldeamanuel:RTE:2018}. Motion planning is central to an autonomous vehicle, which is responsible for identifying the trajectories and maneuvers of the vehicle from a starting configuration to a goal configuration,   while avoiding potential collisions and minimizing certain costs. 

Various motion planning approaches have been proposed, ranging from graph-based methods such as Dijkstra and A* search~\cite{Dijkstra:NM:1959, Hart:TSSC:1968} to sample-based approaches such as   rapidly-exploring random trees and probabilistic roadmaps~\cite{LaValle:ARR:1998, Kavraki:TRA:1996}. However, they are often only suitable for small configuration spaces since their search times will increase exponentially with the dimension of the configuration space. 
Nonlinear model predictive control (NMPC) has emerged as a powerful alternative. By design, NMPC predicts the future behavior of a dynamic system and then finds the best control actions to optimize certain performance indices under input and state constraints. When applied to autonomous vehicles, NMPC can provide two crucial advantages~\cite{Turri:ITSC:2013, Murgovski:CDC:2015, Ji:TVT:2017, Askari:ACC:2021, Askari:ACC:2022}. First, it can continuously and predictively optimize motion plans to handle driving in a changing environment over a receding time horizon.   Second, NMPC can incorporate the vehicle's dynamics, physical limitations, and safety requirements into the online planning procedure, thus improving the safety and feasibility of planning. 
  
The success of NMPC-based motion planning depends on the accuracy of the vehicle's dynamic model. While physics-based vehicle modeling has been in wide use, practitioners often find it challenging to obtain accurate ones~\cite{Draeger:CSM:1995, Piche:NIPS:1999}. Physical models are unable to capture the full range of various uncertain factors acting on a vehicle and often take a long time to develop due to the manual effort of theoretical analysis to experimental calibration. Data-driven machine learning has thus risen in recent years to enable vehicle modeling~\cite{Williams:TRO:2018}. In particular, neural networks have found ever-increasing applications. The popularity of neural networks traces to their universal function approximation properties and their capabilities of using history information to grasp high-order or time-varying effects~\cite{Spielberg:TCST:2022,Hermansdorfer:CIST:2020,Spielberg:SR:2019,Draeger:CSM:1995}.

However, NMPC of neural network models is non-trivial in terms of computational complexity. This is because NMPC must solve a constrained optimization problem in receding horizons. The implementation widely adopts gradient-based optimization solvers, but they struggle to perform a computationally efficient search for optima in the face of the nonlinear nonconvex neural network dynamics~\cite{Nagabandi:ICRA:2018, Draeger:CSM:1995}. The literature has presented some methods to mitigate the issue. A specially structured neural network is developed in~\cite{Broad:arXiv:2018} to facilitate the network's integration with gradient-based NMPC, and an input-convex neural network is used in~\cite{Bunning:PMLR:2021} to avoid nonconvex NMPC, which though comes at a sacrifice of the representation fidelity. The study in~\cite{Nagabandi:ICRA:2018} exploits a random shooting technique for NMPC of neural networks in reinforcement learning tasks.   

We have recently investigated the use of Bayesian estimation tools to solve NMPC of neural networks efficiently. In~\cite{Askari:ACC:2021}, we reformulate an NMPC problem as a Bayesian estimation problem that seeks to infer optimal control actions. Based on the reformulation, we use particle filtering/smoothing to address the estimation problem, exploiting its ability  to perform  sampling-based computation for fast search. Our further study in~\cite{Askari:ACC:2022} applies the method to motion planning based on a neural network vehicle model.
While our prior study shows the promise of using Bayesian estimation to address NMPC, the proposed approach has to perform forward filtering and then backward smoothing to estimate the optimal control actions. Also, we use the bootstrap particle filtering/smoothing technique in~\cite{Askari:ACC:2022}, which is easy to implement but still requires a large number of samples to achieve decent accuracy. These factors would limit the computational performance.  We thus turn our attention to the ensemble Kalman filtering/smoothing.  This technique shares a similar structure with the standard Kalman filtering/smoothing but adopts Monte Carlo sampling-based implementation~\cite{Fang:JAS:2018}.  Its significant computational efficiency and suitability to high-dimensional nonlinear systems  make  it an appealing alternative to solve NMPC within the Bayesian estimation framework.  We will use a sequential ensemble Kalman smoother (EnKS) characterized by a single forward-pass implementation~\cite{Evensen:MWR:2000}. As a recursive forward algorithm, this smoother avoids the two-pass smoothing, i.e., forward filtering and backward smoothing as in~\cite{Askari:ACC:2022}, to further speed up the computation. This will  considerably facilitate motion planning, for which fast, real-time computation is crucial. To sum up, the core contribution of this study lies in developing a novel EnKS-based NMPC approach for autonomous vehicle motion planning. We also present various simulation results to validate the proposed approach.

The rest of the paper is organized as follows. Section~\ref{MP-Formulation} presents the NMPC-based motion planning problem.  Section~\ref{sec:estimation-NMPC}   then reformulates NMPC as a Bayesian estimation problem and develops the EnKS-based motion planning approach. Section~\ref{numerical-sim}  offers simulation results to validate the proposed approach. Finally, Section~\ref{Sec:Conclusion} concludes the paper.
\section{NMPC Motion Planning Problem Formulation}\label{MP-Formulation}

To perform model-based motion planning, we use neural networks to capture the ego vehicle (EV) dynamics. Compared to first-principles vehicle modeling, neural networks are able to learn complex and non-transparent vehicle dynamics from abundant data and show excellent predictive accuracy even in the presence of uncertainty~\cite{Spielberg:SR:2019}. We thus consider a neural network-based vehicle model as follows:
\begin{equation}\label{NN-dynamics}
\dot{x}_{k} = f_{\mathrm {NN}}(x_k,u_k),
\end{equation}
where $f_{\mathrm{NN}}(\cdot)$ represents a feedforward neural network, $x_k$ is the state vector, and $u_k$ is the control input vector.  In above, $x_k = \begin{bmatrix}x_k^p &  y_k^p & \psi_k & \nu_k\end{bmatrix}^{\top}$, where $ \left(x_k^p,\, y_k^p\right)$ is the EV's position in the global coordinates, $\psi_k$ is the heading angle, and $\nu_k$ is the speed; $u_k=\begin{bmatrix}a_k & \delta_k\end{bmatrix}^{\top}$, where $a_k$ is the vehicle's acceleration, and $\delta_k$ is the front wheel's steering angle. We discretize the model in~\eqref{NN-dynamics} as below for the purpose of computation in motion planning:
\begin{align}\label{nn-dynamics-discrete}
x_{k+1} = f(x_k, u_k) = x_k + \Delta t \cdot f_{\mathrm{NN}}(x_k, u_k),
\end{align}
where $\Delta t$ is the sampling period.

The EV should never collide with its surrounding obstacle vehicles (OVs). At time $k$, we denote the state of the EV as $x_k^{\mathrm{EV}}$ and the state of the $i$-th OV as $x_k^{\mathrm{OV}, i}$ for $i=1,\ldots,M_{O}$, where $M_{O}$ is the number of OVs. To prevent collisions, the EV must maintain a safe distance from all OVs. This requirement can be expressed as:
\begin{equation}\label{collision}
d\left(x_k^{\mathrm{EV}}, x_k^{\mathrm{OV}, i}\right) \geq d_{\min}, \quad i=1,\ldots,M_O,
\end{equation}
where $d\left(\cdot, \cdot\right)$ represents the distance between the vehicles characterized by their shape, and $d_{\min}$ is the minimum safe distance. The EV is assumed to be driven on a structured road and expected to stay within the road's boundaries. This requirement is met by imposing the following constraint:
\begin{equation} \label{distance-to-track}
d_B(x_k)\leq L,
\end{equation}
where $d_B(x_k) =  \norm{x_k-B}_2$, $B$ is the closest point from the road boundary to the EV, and $L$ is the road's width. The EV's acceleration and steering control inputs are limited due to both the actuation limits and the need to ensure passenger comfort. This implies
\begin{equation}\label{control-bound}
u_{\textrm{min}}\leq u_k \leq u_{\textrm{max}},
\end{equation}
where $u_{\textrm{min}}$ and $u_{\textrm{max}}$ are the lower and upper control bounds. To simplify the notation, we summarize the planning constraints~\eqref{collision}-\eqref{control-bound} compactly as follows:
\begin{align}\label{inequality-constraints}
g ({x}_k,u_{k})\leq 0. 
\end{align}

With the above formulation, we are now ready to formulate the NMPC motion planning problem. The problem setup should not only include the above constraints but also reflect the general driving requirements. We expect the EV to track a nominal reference path, penalize actuation costs, and comply with the safety constraints. The corresponding NMPC problem is stated as follows:
\begin{subequations} \label{NMPC-Standard}
\begin{align}
\min \quad &  \sum_{t=k}^{k+H}
   (x_t - r_t)^{\top}R(x_t - r_t) +  u_t^{\top}Qu_t,\\
\mathrm{s.t.} \quad & ~\eqref{nn-dynamics-discrete}, \eqref{inequality-constraints},  \quad t = k, \ldots, k+H,
\end{align}
\end{subequations}
where $H$ is the prediction planning horizon length, $Q$ and $R$ are weight matrices, and collectively the optimization variables are $\left\{x_{k:k+H}, \; u_{k:k+H} \right\}$, where $x_{k:k+H} = \left\{ x_k, \ldots, x_{k+H}\right\}$. The reference path $r_{k:k+H}$ is composed of a sequence of waypoints given by a higher decision-making module. The motion planner solves the problem in~\eqref{NMPC-Standard} at every time $k$ to obtain  $u_{k:k+H}^{\ast}$  and applies the first element $u_k^{\ast}$ to drive the EV to the next state $x_{k+1}$. The planner then repeats the procedure in receding horizons. 

Traditional solvers for the problem in~\eqref{NMPC-Standard} are based on gradient-based optimization. However, they would struggle to compute solutions efficiently here.  The neural network-based model will make the optimization landscapes extremely nonlinear and nonconvex, to hinder efficient search. As will be shown in Section~\ref{numerical-sim}, the computational time by gradient-based optimization is almost formidable. Therefore, we will pursue an alternative method  that is  more competent in computation, which translates the NMPC problem in~\eqref{NMPC-Standard} into a Bayesian estimation problem and then harnesses the highly efficient EnKS to perform the estimation.

\section{Motion Planning via EnKS-Based NMPC}\label{sec:estimation-NMPC}
 
This section presents the main results, showing the development of an EnKS-based NMPC for motion planning. 
 
\subsection{Setup of NMPC as Bayesian Estimation}\label{sec:esimation-NMPC-setup}

 Given the NMPC problem in~\eqref{NMPC-Standard}, we can view the control objective as evidence and then infer the best control actions given the evidence. This perspective motivates us to set up a Bayesian estimation problem with equivalence to~\eqref{NMPC-Standard}. To show the estimation problem, we begin by considering the following virtual system: 
\begin{subequations}\label{Virtual-System}
\begin{align}
x_{t+1} &= f(x_t, u_t), \label{dynamic-eq}\\
u_{t+1} &=  w_t,\\
r_t &= x_t+ v_t,
\end{align}
\end{subequations}
for $k \leq t \leq k+H$, where $w_t$ and $v_t$ are additive noises. For now, it is our interest to consider the state estimation problem for~\eqref{Virtual-System} from the viewpoint of probabilistic inference. Following the principle of maximum a posterior estimation, the following problem is of our interest: 
\begin{align}\label{MAP-Formulation}
\max_{x_{k:k+H}, u_{k:k+H}} \log \p(x_{k:k+H}, u_{k:k+H}   |  r_{k:k+H}).
\end{align}
This estimation problem holds an equivalence to the NMPC problem in~\eqref{NMPC-Standard} under some conditions~\cite{Askari:ACC:2022}. 
\begin{theorem}
Assume that $w_t$ and $v_t$ are independent white noises with
$
w_t \sim \mathcal{N}(0,Q^{-1})$ and $v_t \sim \mathcal{N}(0, R^{-1})$.
 Then, the problems in~\eqref{NMPC-Standard} and~\eqref{MAP-Formulation} have the same optima if neglecting~\eqref{inequality-constraints}. 
\end{theorem}
\begin{proof}
By the Markovian property of~\eqref{Virtual-System} and Bayes' rule, we have
\begin{align*}
\p(&x_{k:k+H}, u_{k:k+H}   |  r_{k:k+H}) \nonumber \\
& \qquad \propto \prod_{t=k}^{k+H}\p(r_t | x_t)p(u_t)\cdot \prod_{t=k}^{k+H-1} \p(x_{t+1} | x_{t}, u_{t})p(x_k).\nonumber
\end{align*}
As~\eqref{dynamic-eq} is deterministic  and the initial state $x_k$ is known, the above reduces to
\begin{align*}
\p(x_{k:k+H}, u_{k:k+H}   |&  r_{k:k+H}) \propto \prod_{t=k}^{k+H}\p(r_t | x_t)p(u_t).
\end{align*}
This implies that 
\begin{align*}
\log \p(x_{k:k+H}, u_{k:k+H} | r_{k:k+H})& \\
\propto \sum_{t=k}^{k+H} & \log \p(r_t | x_t) + \log p(u_t).
\end{align*}
As  $\p(r_t  |   x_t)  \sim \mathcal{N}(x_t, R^{-1})$ and $ p(u_t)  \sim \mathcal{N}(0, Q^{-1})$, we have
\begin{align*}
\log \p(r_t | x_t) &\propto -(x_t - r_t)^{\top}R(x_t - r_t), \\
\log p(u_t) &\propto -u_t^{\top}Qu_t.
\end{align*}
Combining the above, we see that problem~\eqref{MAP-Formulation} can be expressed as
\begin{align*}
\min_{x_{k:k+H}, u_{k:k+H}} \sum_{t=k}^{k+H} (x_t - r_t)^{\top}R(x_t - r_t) + u_t^{\top}{Q}u_t,
\end{align*}
which is the same as problem~\eqref{NMPC-Standard} without the constraint in~\eqref{inequality-constraints}. The theorem is thus proven.
\end{proof}

Theorem 1 indicates the viability of leveraging Bayesian estimation to solve the original motion planning problem in~\eqref{NMPC-Standard}, whereby one would estimate the best control actions and motion plans based on the planning objectives.

Proceeding further, we must embed~\eqref{MAP-Formulation} with the constraint  in~\eqref{inequality-constraints}. To this end, we introduce a virtual measurement $z_t$:  
\begin{align}\label{Virtual-measurement-barrier-function}
z_t = \phi(g(x_t, u_t)) + \eta_t.
\end{align}
Here, $\phi(\cdot)$ is a barrier function, which outputs zero when the constraint is satisfied and infinity otherwise, and $\eta_t \sim  \mathcal{N}(0,R_\eta)$  is a small additive noise. It is seen that $z_t$ should take zero to ensure constraint satisfaction. We choose $\phi(\cdot)$ to be a softplus function for the sake of numerical computation:
\begin{align}\label{barrier-function}
\phi(s) = \frac{1}{\alpha}\ln \left( 1 + \exp(\beta s)\right),
\end{align}
where the tunable parameters $\alpha$ and $\beta$ determine how strict to implement the constraint.

For notational simplicity, we rewrite the virtual system in~\eqref{Virtual-System} along with $z_t$ in the following augmented form:
\begin{subequations}\label{Virtual-System-Compact}
\begin{align}\label{Virtual-System-Compact-a}
\bar{x}_{t+1} &= \bar{f}(\bar{x}_t) + \bar{w}_t, \\
{\bar y}_t &=\bar h(  \bar{x}_t ) +\bar v_t, \label{Virtual-System-Compact-b}
\end{align}
\end{subequations}
where 
\begin{align*}
\bar{x}_t&=\begin{bmatrix}x_{t} \\ u_{t}\end{bmatrix}, \: \bar{y}_t= \begin{bmatrix} r_t \\ z_t\end{bmatrix}, \: \bar{w}_t=\begin{bmatrix}0 \\ w_t\end{bmatrix}, \\  \bar{v}_t =\begin{bmatrix}v_t \\ \eta_t\end{bmatrix}, \: \; \bar f(&\bar x_t)=\begin{bmatrix}f(x_t, u_t) \\ 0\end{bmatrix},  \quad \bar h(\bar x_t)=\begin{bmatrix}x_t \\ \phi(g\left( x_t, u_t\right))\end{bmatrix}. 
\end{align*}

 Given~\eqref{Virtual-System-Compact}, we will focus on Bayesian state estimation to find out the best $\bar x^{\ast}$ in line of~\eqref{MAP-Formulation} by considering $\p(\bar x_{k:k+H} | \bar y_{k:k+H})$.  This is known as a smoothing problem.  While various computational algorithms are available to address it, we are particularly interested in ensemble Kalman estimation because of their high computational efficiency and performance in handling nonlinearity. In the sequel, we will show a sequential EnKS algorithm to enable computationally fast estimation for~\eqref{Virtual-System-Compact}. 

\subsection{Sequential Ensemble Kalman Smoothing}\label{sec:EnKS}
Ensemble Kalman estimation traces its origin to Monte Carlo simulation. Characteristically, it approximates concerned probability distributions by ensembles of samples and updates the ensembles in a Kalman-update manner when new data are available. Combining the essences of Monte Carlo sampling and Kalman filtering, it can handle strong nonlinearity in high-dimensional state spaces and compute fast. There are different ways to do smoothing in terms of ensemble Kalman estimation. In this study, we choose a sequential EnKS approach in~\cite{Evensen:MWR:2000} to address the estimation problem in Section~\ref{sec:esimation-NMPC-setup}. This approach sequentially updates the ensembles for all the past states at every time upon the arrival of new data.  As such, it performs smoothing in a single forward pass to save much computation compared to the two pass smoothing methods, thus advantageous for our motion planning estimation task. 

Consider $\p(\bar x_{k:t} | \bar y_{k:t})$ for an arbitrary $t \in \left[k, \ldots,k+H\right]$. For notational convenience, we drop the subscript $k$ and define $\mathcal{X}_t = \bar x_{k:t}$ and $\mathcal{Y}_t = \bar y_{k:t}$. Then, using Bayes' rule, we have the following recursive relation:
\begin{equation}\label{Forward-Smoothing-Recursion}
\p(\mathcal{X}_t  |  \mathcal{Y}_t) \propto \p(\bar y_t  |  \bar{x}_t)\p(\bar{x}_{t}  |  \bar{x}_{t-1})\p(\mathcal{X}_{t-1} | \mathcal{Y}_{t-1}).
\end{equation}
At time $t-1$, we assume

\begin{align}\label{Gaussian_assumption}
p & \left(\left. \begin{bmatrix}
\mathcal{X}_t \\ \bar{y}_{t}
\end{bmatrix} \, \right| \, \mathcal{Y}_{t-1}\right) \sim \nonumber \\  
 & \qquad \quad \; \; \; \; \mathcal{N}\left(\begin{bmatrix}
\mathcal {\hat X}_{t|t-1} \\ \hat{ \bar y}_{t|t-1}
\end{bmatrix}, \begin{bmatrix}
\mathcal{P}_{t|t-1} ^{\mathcal{X}} & \mathcal{P}_{t|t-1}^{\mathcal{X}, \bar{y}} \\ \left(\mathcal{P}_{t|t-1}^{\mathcal{X},\bar {y}} \right)^\top & \mathcal{P}_{t|t-1} ^{\bar{y}}
\end{bmatrix}\right),
\end{align} 
where $\mathcal{\hat X}$ and $\hat{\bar y}$ are the means for $\mathcal{X}$ and $\bar{y}$, respectively, and $\mathcal{P}$ represents the covariances (or cross-covariances). It then follows that
\begin{align}\label{conditional-dist}
\p( \mathcal{X}_{t} | \mathcal{Y}_t) \sim \mathcal{N} \left(\hat {\mathcal{X}}_{t|t} , \mathcal{P}^{\mathcal{X}}_{t|t}\right),
\end{align}
where 
\begin{subequations}\label{KF-update}
\begin{align}\label{KF-update-a}
\hat{\mathcal{X}}_{t|t} &= \hat{\mathcal{X}}_{t|t-1} + \mathcal{P}_{t|t-1}^{\mathcal{X},\bar y} \left( \mathcal{P}_{t|t-1} ^{\bar y} \right)^{-1} \left(\bar y_t -\hat{\bar y}_{t|t-1} \right),\\ \label{KF-update-b}
\mathcal{P}_{t|t}^{\mathcal{X}} &= \mathcal{P}_{t|t-1}^{\mathcal{X}} - \mathcal{P}_{t|t-1}^{\mathcal{X},\bar y} \left( \mathcal{P}_{t|t-1} ^{\bar y} \right)^{-1} \left(\mathcal{P}_{t|t-1}^{\mathcal{X},\bar y} \right)^\top.
\end{align}
\end{subequations}
When $t$ moves from $k$ to $k+H$, the smoothing process is completed for the considered horizon.

To realize the above procedure, we approximately represent $\p(\mathcal{X}_{t-1} | \mathcal{Y}_{t-1})$ by an ensemble of samples $\mathcal{X}^i_{t-1|t-1}$ for $i=1,\ldots, N$ with mean $\mathcal{\hat X}_{t-1|t-1}$ and covariance $\mathcal{P}^{\mathcal{X}}_{t-1|t-1}$. Then, we  pass $\bar x_{t-1|t-1}^i$ through~\eqref{Virtual-System-Compact-a} to obtain the ensemble that approximates $\p(\bar x_t | \mathcal{Y}_{t-1} )$:
\begin{align}\label{predicted-state}
\bar{x}_{t|t-1}^i = \bar{f}(\bar{x}_{t-1|t-1}^i) + \bar{w}_{t-1}^i, \quad i = 1, \ldots, N,
\end{align}
where $\bar{w}_{t-1}^i$ are samples drawn from $\mathcal{N}(0, \mathrm{diag}(0, Q^{-1}))$. 
The corresponding sample mean and covariance can be calculated as
\begin{subequations}\label{Prediction-sample-stats}
\begin{align}
\hat{\bar x}_{t|t-1} &= \frac{1}{N}\sum_{i=1}^N \bar x_{t|t-1}^i,\label{Prediction-sample-mean}\\
\mathcal{P}_{t|t-1}^{\bar x} &= \frac{1}{N-1}\sum_{i=1}^N \left( \bar{x}_{t|t-1}^i - \hat{\bar x}_{t|t-1}\right)\left( \bar{x}_{t|t-1}^i - \hat{\bar{x}}_{t|t-1}\right)^{\top}, \label{Prediction-sample-cov} 
\end{align}
\end{subequations}
Then, we concatenate $\bar x_{t|t-1}^i$ and $\hat{\bar x}_{t|t-1}$ with $\mathcal{X}_{t-1|t-1}^i$ and $\hat{\mathcal{X}}_{t-1|t-1}$, respectively,   to create $\mathcal{X}_{t|t-1}^i \leftarrow (\mathcal{X}_{t-1|t-1}^i, \bar{x}_{t|t-1}^i)$ and $\hat{\mathcal{X}}_{t|t-1} \leftarrow (\hat{\mathcal{X}}_{t-1|t-1}, \hat{\bar{x}}_{t|t-1})$. The covariance for $\p(\mathcal{X}_{t} | \mathcal{Y}_{t-1})$ is constructed as

\begin{align}\label{Prediction-Cov}
\mathcal{P}_{t|t-1}^{\mathcal{X}} = \begin{bmatrix}
\mathcal{P}_{t-1|t-1} ^{\mathcal{X}} & \mathcal{P}_{t|t-1}^{\mathcal{X}, \bar{x}} \\ \left(\mathcal{P}_{t|t-1}^{\mathcal{X},\bar {x}} \right)^\top & \mathcal{P}_{t|t-1} ^{\bar{x}}
\end{bmatrix},
\end{align}
where
\begin{align*}
\mathcal{P}_{\scaleto{t|t-1}{6pt}}^{\mathcal{X}, \bar x} &= \frac{1}{N-1}\sum_{i=1}^N \left( \mathcal{X}_{\scaleto{t-1|t-1}{6pt}}^i - \hat{\mathcal{X}}_{\scaleto{t-1|t-1}{6pt}}\right)\left( \bar{x}_{\scaleto{t|t-1}{6pt}}^i - \hat{\bar{x}}_{\scaleto{t|t-1}{6pt}}\right)^{\top}. 
\end{align*}
To approximate $\p(\bar y_{t} | \mathcal{Y}_{t-1})$, we construct an ensemble of samples by 
\begin{align}\label{ensemble-measurement}
\bar y_{t|t-1}^i = \bar{h}(\bar{x}_{t|t-1}^i) + \bar{v}_t^i, \quad i = 1, \ldots, N,
\end{align}
where $\bar{v}_t^i$ are drawn from $\mathcal{N}\left(0, \mathrm{diag}\left(R^{-1}, R_{\eta}\right)\right)$. The sample mean, and covariance of the measurement ensemble is 
\begin{subequations}
\begin{align}\label{measurement-sats}
\hat{\bar y}_{t|t-1} &= \frac{1}{N}\sum_{i=1}^N \bar{y}_{t|t-1}^i, \\
\mathcal{P}^{\bar y}_{t|t-1} &= \frac{1}{N-1}\sum_{i=1}^N \left(\bar{y}_{t|t-1}^i - \hat{\bar y}_{t|t-1}\right)\left(\bar{y}_{t|t-1}^i - \hat{\bar y}_{t|t-1}\right)^{\top}.\label{covariance_1}
\end{align}
\end{subequations}
\begin{algorithm}[H]
\fontsize{9}{9}
  \caption{\EnKMP: the motion planner based on   NMPC realized by EnKS.} \label{NMPC-EnKS}
  \begin{algorithmic}[1]
\State Set up the NMPC motion planning problem as in~\eqref{NMPC-Standard} 

\State Convert the problem to its Bayesian estimation counterpart as in~\eqref{Virtual-System-Compact}

\For{$k=1, \ldots, T$}


\For{$t = k,\ldots, k+H$}

\If{$t = k$}
  \If{$k = 1$}
    \State Draw samples $\bar x_k^i \sim p(\bar x_k)$, $i = 1, \ldots, N$
  \Else
    \State Warmstart $\bar{x}_k^i$ using $\bar{x}_{t+1}^i$ from $\mathcal{X}_{k+H|k+H}^i$ 
    \Statex \qquad \qquad \quad \; \: at time $k-1$
  \EndIf
\EndIf

\State Generate $\bar{x}_{t|t-1}^i$ using~\eqref{predicted-state} 
\State Compute $\hat{\bar x}_{t|t-1}$ and $\mathcal{P}^{\bar{x}}_{t|t-1}$ via~\eqref{Prediction-sample-mean}-\eqref{Prediction-sample-cov}  
\State Concatenate $\mathcal{X}_{t|t-1}^i \leftarrow (\mathcal{X}_{t-1|t-1}^i, \bar{x}_{t|t-1}^i)$ 
\State Concatenate $\hat{\mathcal{X}}_{t|t-1} \leftarrow (\hat{\mathcal{X}}_{t-1|t-1}, \hat{\bar{x}}_{t|t-1})$ 
\State Compute $\mathcal{P}^{\mathcal{X}}_{t|t-1}$ via~\eqref{Prediction-Cov}
\State Generate $\bar y_{t|t-1}^i$ using~\eqref{ensemble-measurement} 
\State Compute $\hat{\bar y}_{t|t-1}$ and $\mathcal{P}^{\bar y}_{t|t-1}$ via~\eqref{measurement-sats}-\eqref{covariance_1}
\State Compute $\mathcal{P}^{\mathcal{X}, \bar y}_{t|t-1}$ via~\eqref{cross-cov}
\State Obtain updated samples $\mathcal{X}_{t|t}^i$ using~\eqref{ensemble-update}
\State Compute $\hat{\mathcal{X}}_{t|t}$ and $\mathcal{P}^{\mathcal{X}}_{t|t}$ via ~\eqref{update-mean}-\eqref{update-cov}

\EndFor

\State Extract the estimated control $u_k^\ast$ from $\hat{\mathcal{X}}_{k+H|k+H}$

\EndFor

\end{algorithmic}
\label{EnKS-NMPC-MP}
\end{algorithm}
\vspace{-0.5em}
\noindent Then, the cross-covariance between $\mathcal{X}_{t}$ and $\bar y_t$ is found by
\begin{align}\label{cross-cov}
\mathcal{P}^{\mathcal{X},\bar y}_{t|t-1} &= \frac{1}{N-1}\sum_{i=1}^N \left( \mathcal{X}_{t|t-1}^i - \hat{\mathcal{X}}_{t|t-1}\right)\left(\bar y_{t|t-1}^i - \hat{\bar y}_{t|t-1}\right)^{\top}.
\end{align}
Based on~\eqref{KF-update}, $\mathcal{X}_{t|t-1}^i$ can be updated individually as
\begin{align}\label{ensemble-update}
\mathcal{X}_{t|t}^i = \mathcal{X}_{t|t-1}^i + \mathcal{P}_{t|t-1}^{\mathcal{X},\bar y} \left(\mathcal{P}_{t|t-1} ^{\bar y} \right)^{-1} \left(\bar{y}_t- \bar{y}_{t|t-1}^i \right). 
\end{align}
Finally, the smoothed estimate and the associated covariance are given by
\begin{subequations}\label{trajectory-estimate}
\begin{align}
\hat{\mathcal{X}}_{t|t} &= \frac{1}{N}\sum_{i=1}^N \mathcal{X}_{t|t}^i, \label{update-mean} \\ 
\mathcal{P}^{\mathcal{X}}_{t|t} &=   \frac{1}{N-1}\sum_{i=1}^N \left( \mathcal{X}_{t|t}^i - \hat{\mathcal{X}}_{t|t}\right)    \left( \mathcal{X}_{t|t}^i - \hat{\mathcal{X}}_{t|t}\right)^{\top}. \label{update-cov}
\end{align}
\end{subequations}

The EnKS is outlined in~\eqref{predicted-state}-\eqref{trajectory-estimate}. This approach is characterized by sequential computation in one pass. Compared to two-pass smoothers in the literature, it presents higher computational efficiency and thus is a favorable choice for addressing our estimation problem for motion planning. Note that the execution of~\eqref{Prediction-sample-cov},~\eqref{Prediction-Cov}, and~\eqref{update-cov} can be skipped to speed up the computation further if one has no interest in computing the covariances for the quantification of the uncertainty in estimation.  Based on the outlined EnKS approach, we can solve the NMPC-based motion planning problem and propose the \EnKMP as summarized in Algorithm~\ref{EnKS-NMPC-MP}.
\begin{figure}[t!]
\includegraphics[width=0.49\textwidth, trim={10cm 0.5cm 10cm 1cm},clip]{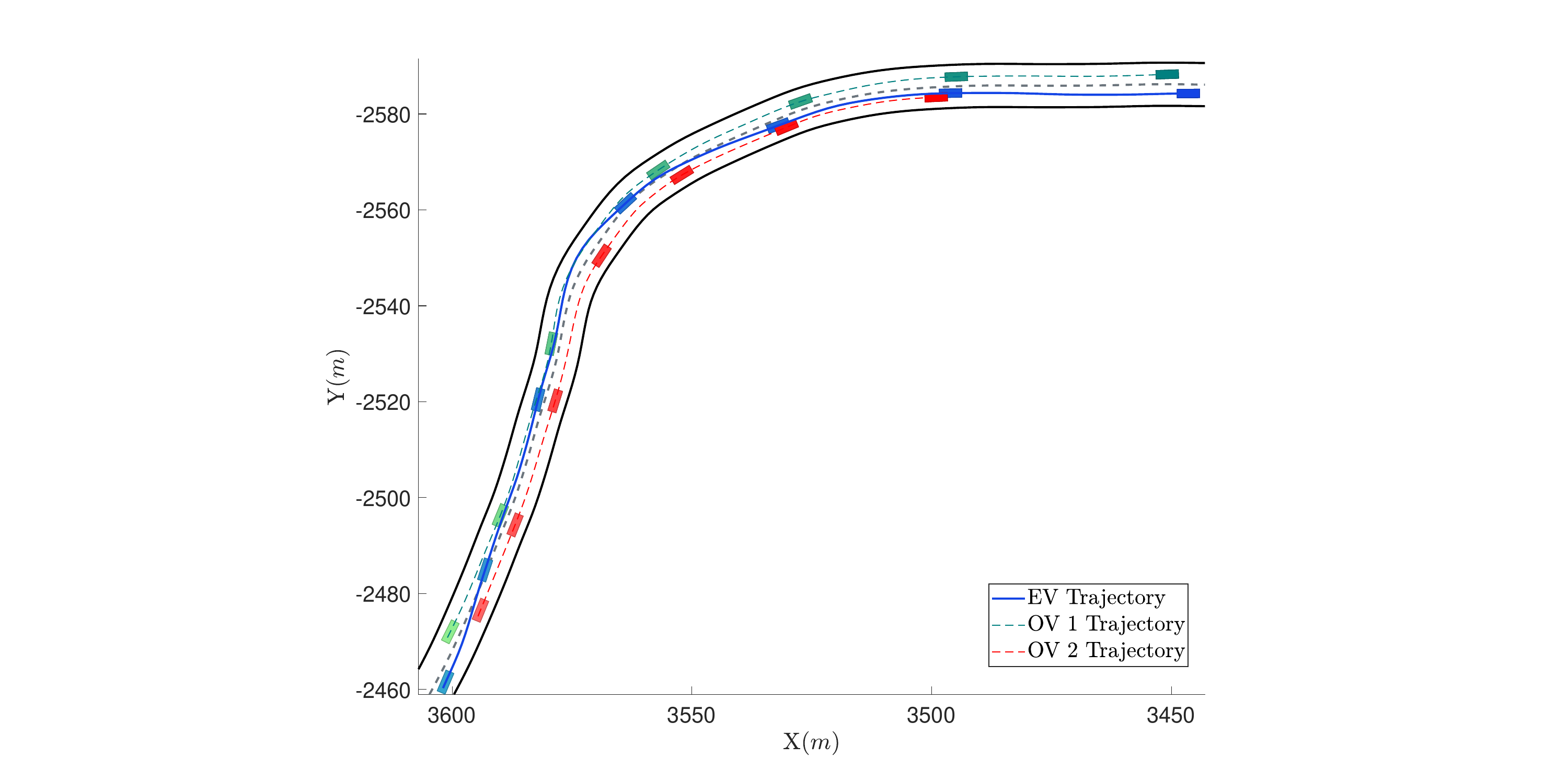}
\centering
 \caption{Vehicle trajectories and positions during the simulation. The EV  performing EnKS motion planning is denoted in blue. The other green and red vehicles are OVs. \iffalse The solid blue line represents the trajectory of the EV, while the dashed lines show the trajectories of the OVs. \fi The color gradient from light to dark represents the vehicle's position from the past to the future.}
 \label{fig:Trajectory}
 \vspace{-1.5em}
\end{figure}


\begin{figure*}[t]
	    \centering
    \subfloat[\centering ]{{\includegraphics[trim={4cm 8cm 3.6cm 9.1cm},clip,width=5cm]{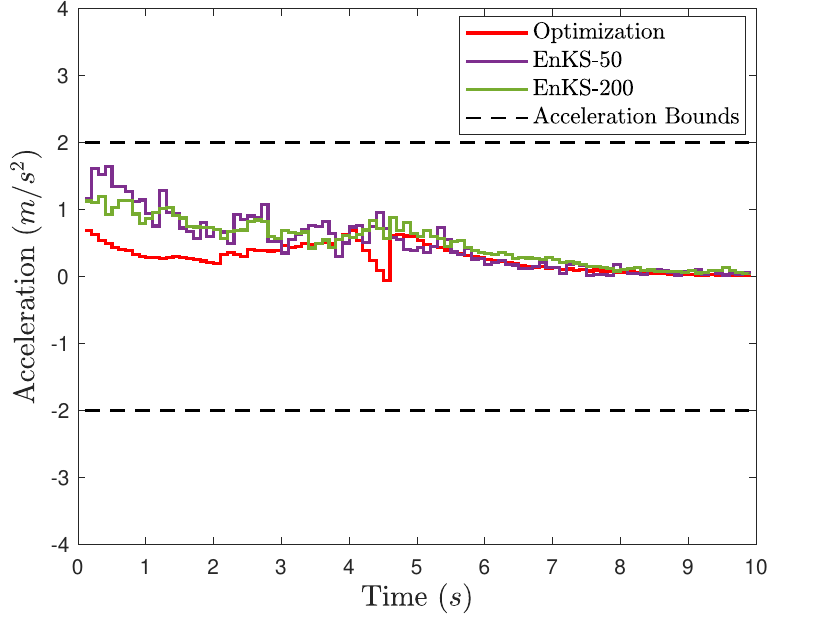} }\label{fig:controls-a}}
    \subfloat[\centering ]{{\includegraphics[trim={4cm 8cm 3.6cm 9.1cm},clip,width= 5cm]{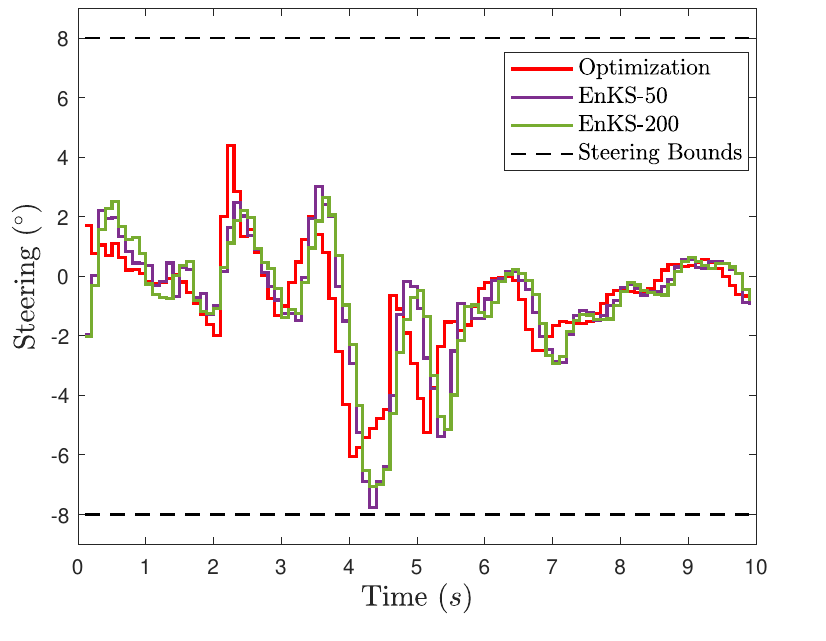} }\label{fig:controls-b}}
    \subfloat[\centering ]{{\includegraphics[trim={4cm 8cm 3.8cm 9.1cm},clip,width=4.9cm]{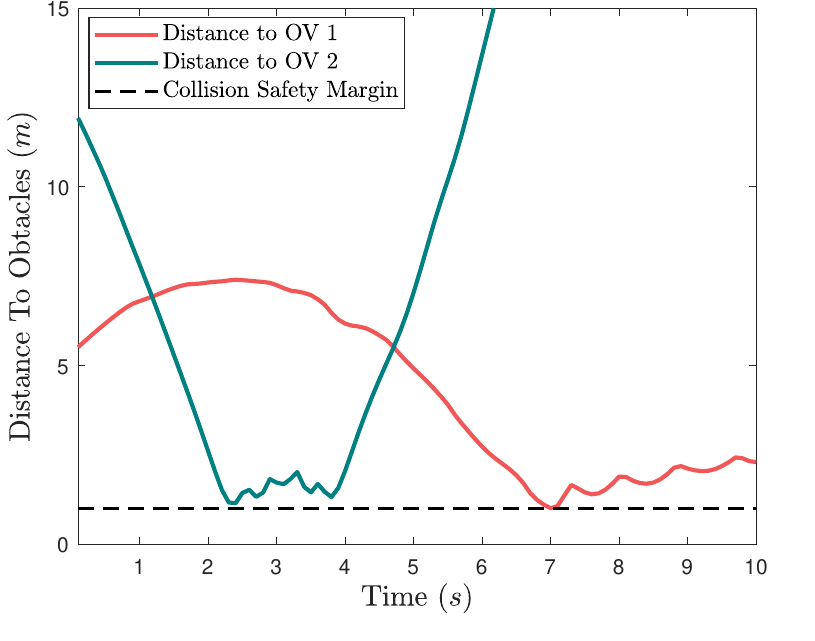} }\label{fig:controls-c}}
    \caption{The control profile and constraint satisfaction by the EV running optimization and \EnKMP. (a) Acceleration control profiles with respective bounds. (b) Steering control profiles with respective bounds. (c) Distance between the EV and OVs when $N=200$ with a safety margin of 1 m.}
    \label{fig:Controls}
    \vspace{-1.7em}
\end{figure*}

\section{Numerical Simulation}\label{numerical-sim}

In this section, we apply the \EnKMP to motion planning when the EV is set to overtake the slow-moving OVs on a curved road. The trained neural network consists of two hidden layers with 128 neurons in each layer and is trained using the Adam optimizer. We generated the training data using the single-track bicycle model in~\cite{Rajmani:Springer:2012}. The sampling period $\Delta t$ is $0.1 \, \textrm{s}$.  

 \begin{figure*}[t!]
        \centering
    \subfloat[\centering ]{{\includegraphics[trim={3cm 8cm 2.5cm 8.2cm},clip,width=6cm]{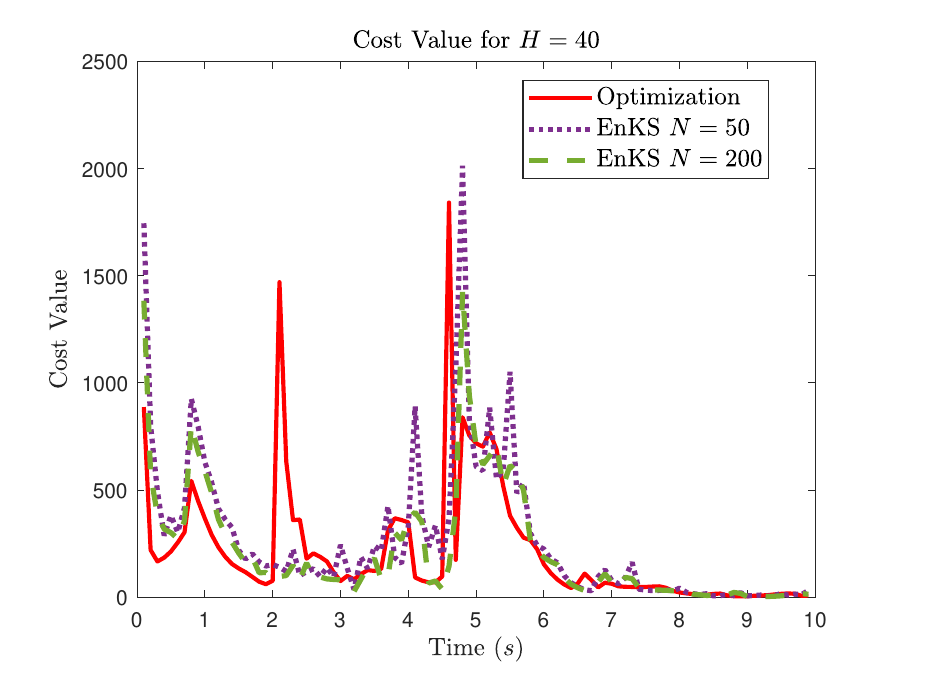} }\label{fig:Cost-H40}} 
    \subfloat[\centering ]{{\includegraphics[trim={3.5cm 8cm 2.5cm 8.2cm},clip,width= 5.9cm]{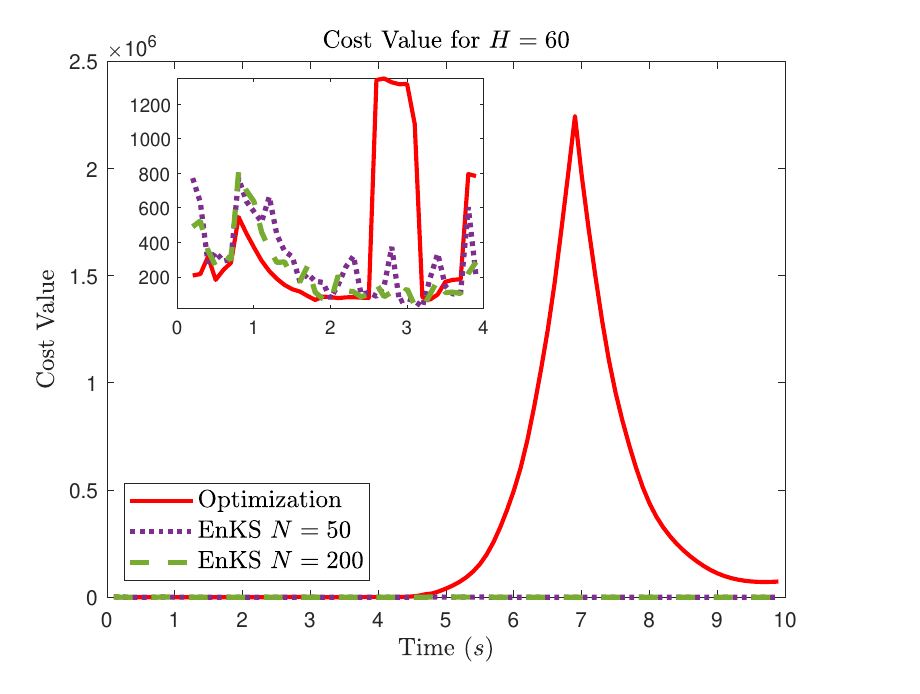} }\label{fig:Cost-H60}} 
    \caption{Cost comparison over the simulation time between optimization and EnKS with different particle numbers for two horizon lengths. (a) $H = 40$. (b) $H = 60$.}
    \label{fig:Cost}
    \vspace{-1.7em}
\end{figure*}

In Fig.~\ref{fig:Trajectory}, we depict the movement of the blue-colored EV and two OVs in green and red. All the vehicles start from the left and travel to the right along the curved road. The motion planner must help the EV   find the best path and actuation to accomplish the overtaking task safely on the curved road. The EV can always find a safe trajectory forward using the proposed \EnKMP, as shown in Fig.~\ref{fig:Controls}. We also evaluate the effects of the ensemble size $N$ on the estimated control profiles, as illustrated in Figs.~\ref{fig:controls-a}-\ref{fig:controls-b}. The increase in sample size results in more accurate estimation and smoother control profiles while meeting the driving constraints, as is seen. Further, the distances between the EV and OVs are plotted in Fig.~\ref{fig:controls-c} when $N=200$, depicting the capability of the \EnKMP to always maintain a safe distance between the EV and OVs.     

As discussed earlier, the \EnKMP is designed for fast computation. We thus evaluate its computational performance as well as its cost performance through a comparison with gradient-based optimization. The simulations run on a workstation equipped with an Intel i9-10920X 3.5 GHz CPU and 128 GB of RAM.  For gradient-based optimization, we used MATLAB's 2023a MPC Toolbox with the default interior-point algorithm and computed the expressions of the gradients offline for the optimization run. 

In Fig.~\ref{fig:Cost}, we compare the cost value for the \EnKMP and gradient-based optimization motion planner for $H = 40$ and $H=60$. When $H = 40$, the \EnKMP has comparable cost performance with gradient-based optimization. On the contrary, when $H = 60$, gradient-based optimization fails due to the increase in nonconvexity and size of the optimization problem, while the \EnKMP is successful because of its sequential formulation and ability to handle high-dimensional nonlinear Bayesian state estimation problems.  
 
In Table~\ref{Table: Computation-Comp}, we report the numerical comparison between both approaches when the prediction horizon $H=40$ and $ 60$ and when the \EnKMP's ensemble size $N$ is $50$, $100$ and $200$. Additionally, we have numerically summarized the cost attained by each approach during the simulation time interval $0\leq k \leq 500$. The relative cost achieved by the \EnKMP compared to gradient optimization when $H  = 40$ and $N = 200$ is $2 \%$ higher than optimization, which is acceptable considering the computational efficiency. In all the cases, the \EnKMP has remarkable computational efficiency, which is faster by orders of magnitude compared with gradient optimization, having a relative decrease of over $99 \%$. Again, note that gradient optimization fails when $H=60$ and requires almost unmanageable computational times even with shorter prediction horizons from a practical viewpoint. By contrast, the \EnKMP still demonstrates very fast speed, showing its desired capability of enabling long prediction horizons and good scalability with the number of optimization variables involved in planning.
\begin{table}[t!]\large\centering 
\caption{Numerical comparison of EnKS and optimization}
\resizebox{0.5\textwidth}{!}{
 \begin{tabular}{c  c  c  c  c  c}
\toprule
\makecell[c]{Horizon \\ ($H$)} & Method & Total Cost & \makecell[c]{ Average \\ Computation \\ Time (s)} & \makecell[c]{Relative \\ Cost Change \\ (\%)} & \makecell[c]{Relative \\ Computation \\ Time Change\\ (\%)}  \\
\midrule
 \multirow{5}{0.5cm}{\makecell[c]{$40$}} & Optimization & $14,680$ & $54.92$ & --- & ---  \\
 & EnKS ($N = 50$) & $17,853$ & $0.158$ & $21.6$ & $-99.71$ \\
 & EnKS ($N = 100$) & $16,126$ & $0.242$ & $9.85$ & $-99.56$ \\
 & EnKS ($N = 200$)& $15,102$ & $0.366$ & $2.88$ & $-99.61$ \\
 \cmidrule(l){1-6}
 \multirow{5}{0.5cm}{\makecell[c]{$60$}} & Optimization & --- & $143.3$ & --- & ---  \\
 & EnKS ($N = 50$) & $18,845$ & $0.250$ & --- &  $-99.83$\\
 & EnKS ($N = 100$) & $17,655$ & $0.385$ & --- &  $-99.73$\\
 & EnKS ($N = 200$) & $15,798$ & $0.561$ & --- &  $-99.61$\\
\bottomrule
\end{tabular}
 }
\label{Table: Computation-Comp}
\vspace{-1.2em}
\end{table}

\section{Conclusion}\label{Sec:Conclusion}


Autonomous vehicles rely on high-performance motion planning to achieve their full potential. In this paper, we consider the design of a motion planner that uses neural networks to predict a vehicle's dynamics and exploits NMPC to compute safe, optimal trajectories. However, the planner will suffer from high computational complexity if adopting gradient-based solvers, as the neural network model will make the underlying optimization problem highly nonlinear and nonconvex. We show that a sequential EnKS approach can approximately implement the NMPC-based motion planner with very fast computation.  What makes this design possible is the reformulation of NMPC as a Bayesian estimation problem. In the simulations, the proposed \EnKMP shows itself to be faster than gradient-based optimization by orders of magnitude.

\balance
\bibliographystyle{IEEEtran}
\bibliography{ACC24-Bib}
\end{document}